\documentclass[runningheads]{llncs}
\usepackage[english]{babel}
\usepackage[utf8x]{inputenc}
\usepackage[T1]{fontenc}
\usepackage{graphicx}
\usepackage{amsfonts,amssymb,amsmath}
\usepackage[ruled,linesnumbered,vlined]{algorithm2e}
\usepackage[noabbrev,capitalise]{cleveref}
\usepackage[misc]{ifsym} 
\usepackage{siunitx} 
\usepackage[caption=false]{subfig} 
\spnewtheorem{assumption}{Assumption}{\bfseries}{\itshape}
\crefname{assumption}{Assumption}{Assumptions}
\usepackage[dvipsnames]{xcolor}
\usepackage{tikz}
\usetikzlibrary{shapes,arrows,calc,spy}
\usepackage{lscape}
\usepackage{pgfplots}
\pgfplotsset{compat=newest}
\usetikzlibrary{plotmarks}
\usetikzlibrary{arrows.meta}
\usepgfplotslibrary{patchplots}
\usepackage{grffile}
\definecolor{mycolor1}{rgb}{0.00000,0.44700,0.74100}%
\definecolor{mypurple}{rgb}{0.5961,0.3059,0.6392}%
\definecolor{mygreen}{rgb}{0.7020,0.8706,0.4118}%
\definecolor{myorange}{rgb}{0.9843,0.5020,0.4471}%
\graphicspath{{./img/}}
\DeclareMathOperator*{\argmin}{\arg\!\min}
\DeclareMathOperator*{\argmax}{\arg\!\max}
\def\arbo{{\sc{arborescent}}}
\def\brassens{{\sc{bras\-sens}}}
\DeclareMathOperator{\col}{col} 
\DeclareMathOperator{\conv}{conv} 
\DeclareMathOperator{\rank}{rank}
\usepackage{todonotes}

\definecolor{brightpink}{rgb}{1.0, 0.0, 0.5}

\begin{document}
\title{Sparse Separable Nonnegative Matrix Factorization}
%
%
\author{Nicolas Nadisic\inst{1}\Letter{} \and
  Arnaud Vandaele\inst{1} \and
  Jeremy E. Cohen\inst{2} \and
  Nicolas Gillis\inst{1}}
%
\authorrunning{N. Nadisic et al.} 
%
\institute{University of Mons, Belgium.\\
\email{\{nicolas.nadisic, arnaud.vandaele, nicolas.gillis\}@umons.ac.be} \and Univ Rennes, Inria, CNRS, IRISA, Rennes, France.\\
\email{jeremy.cohen@irisa.fr}
}
\maketitle              
%
%
%
\begin{abstract}

We propose a new variant of nonnegative matrix factorization (NMF), combining separability and sparsity assumptions. 
Separability requires that the columns of the first NMF factor are equal to columns of the input matrix, while sparsity requires that the columns of the second NMF factor are sparse. 
We call this variant sparse separable NMF (SSNMF), which we prove to be NP-complete, as opposed to separable NMF which can be solved in polynomial time.  
The main motivation to consider this new model is to handle underdetermined blind source separation problems, such as multispectral image unmixing.  
We introduce an algorithm to solve SSNMF, based on the successive nonnegative projection algorithm (SNPA, an effective algorithm for separable NMF), and an exact sparse nonnegative least squares solver.
We prove that, in noiseless settings and under mild assumptions, our algorithm recovers the true underlying sources.
This is illustrated by experiments on synthetic data sets and the unmixing of a multispectral image.

\keywords{Nonnegative Matrix Factorization  \and Sparsity \and Separability.}
\end{abstract}

\section{Introduction}
\label{sec:intro}

Nonnegative Matrix Factorization (NMF) is a low-rank model widely used for feature extraction in applications such as multispectral imaging, text mining, or blind source separation; see~\cite{gillis_why_2014,fu2019nonnegative} and the references therein. 
Given a nonnegative data matrix $M \in \mathbb{R}_+^{m \times n}$ and a factorization rank $r$, NMF consists in finding two nonnegative matrices $W \in \mathbb{R}^{m \times r}_+$ and $H \in \mathbb{R}^{r \times n}_+$ such that $M \approx W H$.
NMF can be formalized as the following optimization problem:
\begin{equation}
    \min_{W \geq 0, H \geq 0} \| M - WH \|_F^2 .
\end{equation}
In this paper, we use the Frobenius norm to measure the quality of the approximation.
Although other measures are possible, the Frobenius norm is by far the most commonly used, because it assumes Gaussian noise (which is reasonable in many real-life applications) and allows for efficient computations~\cite{gillis_why_2014}.

One of the advantages of NMF over similar methods such as principal component analysis (PCA) is that the nonnegativity constraint favors a part-based representation~\cite{lee1999learning}, which is to say that the factors are more easily interpretable, in particular when they have a physical meaning.
If each column of $M$ represents a data point, then each corresponding column of $H$ contains the coefficients to reconstruct it from the $r$ atoms represented by the columns of $W$, since $M(:,j) \approx WH(:,j)$ for all $j$.
Every data point is therefore expressed as a linear combination of atoms.
For example, when using NMF for multispectral unmixing, a data point is a pixel, an atom is a specific material, and each column of $H$ contains the abundance of these materials in the corresponding pixel; see Section~\ref{sec:xp-real} for more details. 
Geometrically, the atoms (columns of $W$) can be seen as $r$ vertices whose convex hull contains the data points (columns of $M$), under appropriate scaling.

\subsection{Separability}

In general, computing NMF is NP-hard~\cite{vavasis_complexity_2010}. 
However, Arora et al.~\cite{arora_computing_2012} proved that NMF is solvable in polynomial time under the \emph{separability}  assumption on the input matrix. 
\begin{definition}\label{def:separability}
A matrix $M$ is $r$-separable if there exists a subset of $r$ columns of $M$, indexed by $\mathcal{J}$, and a nonnegative matrix $H \geq 0$, such that $M = M(:,\mathcal{J}) H$.
\end{definition}
Equivalently, $M$ is $r$-separable if $M$ has the form 
$M = W [ I_r , H' ] \Pi$, where $I_r$ is the identity matrix of size $r$, $H'$ is a nonnegative matrix, and $\Pi$ is a permutation. 
Separable NMF consists in selecting the right $r$ columns of $M$ such that $M$ can be reconstructed perfectly.
In other words, it consists in finding the atoms (columns of $W$) \emph{among} the data points (columns of $M$).
\begin{problem}[Separable NMF]
Given a $r$-separable matrix $M$, find $W = M(:,\mathcal{J})$ with $|\mathcal{J}| = r$  and $H \geq 0$ such that $M = W H$.
\end{problem}
Note that, if $W$ is known, the computation of $H$ is straightforward: it is a convex problem that can be solved using any nonnegative least squares (NNLS) solver (for example, it can be solved with the Matlab function \texttt{lsqnonneg}). 
However, the solution is not necessarily unique, unless $W$ is full rank.

In the presence of noise, which is typically the case in real-life applications, this problem is called near-separable NMF and is also solvable in polynomial time given that the noise level is sufficiently small~\cite{arora_computing_2012}. In this case, we are given a near-separable matrix $M \approx M(:,\mathcal{J}) H$ where $|\mathcal{J}| = r$ and $H \geq 0$.

\subsection{Successive Nonnegative Projection Algorithm}

Various algorithms have been developed to tackle the (near-)separable NMF problem.
Some examples are the successive projections algorithm (SPA)~\cite{araujo_successive_2001}, the fast canonical hull algorithm~\cite{kumar_fast_2013}, or the successive nonnegative projections algorithm (SNPA)~\cite{gillis_successive_2014}. 
Such algorithms start with an empty matrix $W$ and a residual matrix $R = M$, and then alternate between two steps: a greedy selection of one column of $R$ to be added to $W$, and an update of $R$ using $M$ and the  columns extracted so far.  
As SNPA was shown, both theoretically and empirically, to perform better and to be more robust than its competitors~\cite{gillis_successive_2014}, it is the one we study here in detail. Moreover, SNPA is able to handle the underdetermined case when $\rank(W) < m$ which will be key for our problem setting (see below for more details). 

SNPA is presented in \cref{alg:snpa}.
\begin{algorithm}[tp]
  \caption{SNPA}\label{alg:snpa}

  \KwIn{A near-separable matrix ${M}  \in \mathbb{R}^{m \times n}$, the number $r$ of columns to be extracted, and a strongly convex function $f$ with $f(0)=0$ (by default, $f(x) = \| x \|_2^2$).}

  \KwOut{A set of $r$ indices $\mathcal{J}$, and a matrix $H  \in \mathbb{R}^{r \times n}_{+}$ such that ${M} \approx {M}(:,\mathcal{J}) H$.} 

  \vspace{0.2cm}
  \DontPrintSemicolon
  
  Init $R \gets {M}$\\
  Init $\mathcal{J} = \{\}$\\
  Init $t = 1$
  
  \While{$R \neq 0 \And{} t \leq r$}
  {
    $ p = \argmax_j f (R(:,j))$\\
    $\mathcal{J} = \mathcal{J} \cup \{ p \}$\\
    
    \ForEach{j}
    {
      $H^{*}(:,j) = \argmin\limits_{h \in \Delta} f({M}(:,j) - {M}(:,\mathcal{J}) h)$\\\label{alg:snpa-proj}
      $ R(:,j) = {M}(:,j) - {M}(:,\mathcal{J}) H^{*}(:,j) $ \\
    }
    t = t + 1
  }
\end{algorithm}
SNPA selects, at each step, the column of $M$ maximizing a function $f$ (which can be any strongly convex function such that $f(0)=0$, and $f = ||.||_2^2$ is the most common choice).
Then, the columns of $M$ are projected onto the convex hull of the origin and the columns extracted so far, see step 8 where we use the notation 
\[
\Delta = \Big\{ h \ \big| \ h \geq 0, \sum_i h_i \leq 1 \Big\}, 
\]
whose dimension is clear from the context. 
After $r$ steps, given that the noise is sufficiently small and that the columns of $W$ are vertices of $\conv(W)$, SNPA is guaranteed to identify $W$. 
An important point is that SNPA requires the columns of $H$ to satisfy $\|H(:,j)\|_1 \leq 1$ for all $j$, where $\|x\|_1 = \sum_i |x_i|$ is the $\ell_1$ norm. This assumption can be made without loss of generality by properly scaling the columns of the input matrix to have unit $\ell_1$ norm; see the discussion in~\cite{gillis_successive_2014}.

\subsection{Model Limitations}

Unfortunately, some data sets cannot be handled successfully by separable NMF, even when all data points are linear combinations of a subset of the input matrix. 
In fact, in some applications, the columns of the basis matrix $W$, that is, the atoms, might not be vertices of $\conv(W)$.
This may happen when one seeks a matrix $W$ which is not full column rank.
For example, in multispectral unmixing, $m$ is the number of spectral bands which can be smaller than $r$, which is the number of materials  present in the image; see Section~\ref{sec:xp-real} for more details. 
Therefore, it is possible for some columns of $W$ to be contained in the convex hull of the other columns, that is, to be additive linear combinations of others columns of $W$; see \cref{fig:examples} for illustrations in three dimensions (that is, $m = 3$).
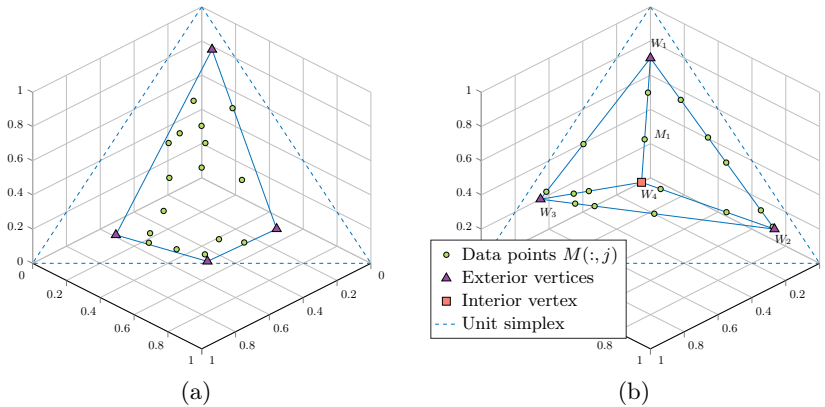
\begin{figure}[t]
    \centering
    \subfloat[\label{fig:allext}]{%
%
%
%
\begin{tikzpicture}[scale=0.5]

\begin{axis}[%
width=3.541in,
height=3.583in,
at={(0.594in,0.484in)},
scale only axis,
xmin=0,
xmax=1,
tick align=outside,
ymin=0,
ymax=1,
zmin=0,
zmax=1,
view={135}{35.2643896827547},
axis background/.style={fill=white},
axis x line*=bottom,
axis y line*=left,
axis z line*=left,
xmajorgrids,
ymajorgrids,
zmajorgrids,
]
\addplot3[only marks, mark=*, mark options={}, mark size=2pt, draw=black, fill=mygreen] table[row sep=crcr]{%
x	y	z\\
0.402129827750314	0.503849733367964	0.0940204388817217\\
0.429260773024408	0.237577429891513	0.333161797084079\\
0.311707558079109	0.181316656478482	0.506975785442409\\
0.616315685064004	0.303404179364213	0.0802801355717832\\
0.473228051783283	0.491824675240473	0.0349472729762437\\
0.333739933493538	0.584942121404151	0.0813179451023111\\
0.106557960301275	0.288560265735927	0.604881773962798\\
0.218577901052293	0.456304497840394	0.325117601107313\\
0.232680956320608	0.231371392613176	0.535947651066216\\
0.254912098483247	0.276231678564715	0.468856222952038\\
0.363506705878983	0.167563968924033	0.468929325196984\\
0.207425110510849	0.159066276013157	0.633508613475994\\
0.594100742163954	0.288595154726627	0.117304103109419\\
0.510634736336602	0.285127697430668	0.204237566232729\\
0.313740706740464	0.313049687867409	0.373209605392127\\
0.546934584618126	0.39821462464855	0.0548507907333236\\
};

\addplot3[only marks, mark=triangle*, mark options={}, mark size=4pt, color=black, fill=mypurple] table[row sep=crcr]{%
x	y	z\\
0.0527258059095908	0.113726883529871	0.833547310560538\\
0.69860062739595	0.19084781775927	0.110551554844779\\
0.211505632603685	0.654410431690122	0.134083935706193\\
0.479345263627531	0.513035401733014	0.00761933463945425\\
};

\addplot3 [color=mycolor1, dashed]
 table[row sep=crcr] {%
0	0	1\\
0	1	0\\
};

\addplot3 [color=mycolor1, dashed]
 table[row sep=crcr] {%
0	0	1\\
1	0	0\\
};

\addplot3 [color=mycolor1, dashed]
 table[row sep=crcr] {%
0	1	0\\
1	0	0\\
};

\addplot3 [color=mycolor1]
 table[row sep=crcr] {%
0.69860062739595	0.19084781775927	0.110551554844779\\
0.479345263627531	0.513035401733014	0.00761933463945425\\
};


\addplot3 [color=mycolor1]
 table[row sep=crcr] {%
0.69860062739595	0.19084781775927	0.110551554844779\\
0.0527258059095908	0.113726883529871	0.833547310560538\\
};

\addplot3 [color=mycolor1]
 table[row sep=crcr] {%
0.479345263627531	0.513035401733014	0.00761933463945425\\
0.211505632603685	0.654410431690122	0.134083935706193\\
};


\addplot3 [color=mycolor1]
 table[row sep=crcr] {%
0.211505632603685	0.654410431690122	0.134083935706193\\
0.0527258059095908	0.113726883529871	0.833547310560538\\
};

\end{axis}

\end{tikzpicture}%
%
    }\hfil
    \subfloat[\label{fig:intvertex}]{%
%
%
%
\begin{tikzpicture}[scale=0.5]

\begin{axis}[%
width=3.541in,
height=3.583in,
at={(0.594in,0.484in)},
scale only axis,
xmin=0,
xmax=1,
tick align=outside,
ymin=0,
ymax=1,
zmin=0,
zmax=1,
view={135}{35.2643896827547},
axis background/.style={fill=white},
axis x line*=bottom,
axis y line*=left,
axis z line*=left,
xmajorgrids,
ymajorgrids,
zmajorgrids,
legend style={at={(-0.15,0.03)}, anchor=south west, legend cell align=left, align=left, draw=white!15!black, nodes={scale=1.5, transform shape}}
]
\addplot3[only marks, mark=*, mark options={}, mark size=2pt, draw=black, fill=mygreen] table[row sep=crcr]{%
x	y	z\\
0.32520138422833	0.385141964842912	0.289656650928758\\
0.465630197826768	0.0695308168477693	0.464838985325462\\
0.0703272644254425	0.723127447065707	0.206545288508851\\
0.0671279775274921	0.790312471922667	0.142559550549841\\
0.275520164796942	0.241104446209306	0.483375388993752\\
0.591018030848994	0.137358562573425	0.27162340657758\\
0.668270183468812	0.052644151377599	0.279085665153589\\
0.553583839969188	0.223387557931225	0.223028602099587\\
0.0918636955979726	0.270862392442576	0.637273911959451\\
0.17460222331209	0.159974336388151	0.665423440299758\\
0.0845324899550969	0.424817710942965	0.490649799101938\\
0.079636620163263	0.527630976571477	0.39273240326526\\
0.541310984631405	0.177203099620858	0.281485915747737\\
0.392029367779484	0.414702064471663	0.193268567748852\\
0.175676656665297	0.625076992792895	0.199246350541808\\
0.606179066187951	0.161103737409005	0.232717196403044\\
};
\addlegendentry{Data points $M(:,j)$}

\addplot3[only marks, mark=triangle*, mark options={}, mark size=4pt, color=black, fill=mypurple] table[row sep=crcr]{%
x	y	z\\
0.0666666666666667	0.8	0.133333333333333\\
0.1	0.1	0.8\\
0.7	0.05	0.25\\
};
\addlegendentry{Exterior vertices}

\addplot3[only marks, mark=square*, mark options={}, mark size=3pt, color=black, fill=myorange] table[row sep=crcr]{%
x	y	z\\
0.368421052631579	0.315789473684211	0.315789473684211\\
};
\addlegendentry{Interior vertex}

\node[right, align=right]
at (axis cs:0.15,0.126,0.924) {$W_1$};
\node[right, align=right]
at (axis cs:0.165,0.874,0.161) {$W_2$};
\node[right, align=right]
at (axis cs:0.807,0.127,0.266) {$W_3$};
\node[right, align=right]
at (axis cs:0.474,0.391,0.335) {$W_4$};
\node[right, align=right]
at (axis cs:0.318,0.314,0.568) {$M_1$};

\addplot3 [color=mycolor1, dashed]
 table[row sep=crcr] {%
0	0	1\\
0	1	0\\
};
 \addlegendentry{Unit simplex}

\addplot3 [color=mycolor1, dashed]
 table[row sep=crcr] {%
0	0	1\\
1	0	0\\
};

\addplot3 [color=mycolor1, dashed]
 table[row sep=crcr] {%
0	1	0\\
1	0	0\\
};

\addplot3 [color=mycolor1]
 table[row sep=crcr] {%
0.1	0.1	0.8\\
0.0666666666666667	0.8	0.133333333333333\\
};

\addplot3 [color=mycolor1]
 table[row sep=crcr] {%
0.1	0.1	0.8\\
0.7	0.05	0.25\\
};

\addplot3 [color=mycolor1]
 table[row sep=crcr] {%
0.1	0.1	0.8\\
0.368421052631579	0.315789473684211	0.315789473684211\\
};

\addplot3 [color=mycolor1]
 table[row sep=crcr] {%
0.0666666666666667	0.8	0.133333333333333\\
0.7	0.05	0.25\\
};

\addplot3 [color=mycolor1]
 table[row sep=crcr] {%
0.0666666666666667	0.8	0.133333333333333\\
0.368421052631579	0.315789473684211	0.315789473684211\\
};

\addplot3 [color=mycolor1]
 table[row sep=crcr] {%
0.7	0.05	0.25\\
0.368421052631579	0.315789473684211	0.315789473684211\\
};

\end{axis}

\end{tikzpicture}%
%
    }
    \caption{On the left (a), all vertices are exterior, and SNPA is assured to identify them all. On the right (b), the data points are 2-sparse combinations of 4 points, one of which (vertex 4) is ``interior'' hence it cannot be identified with separable NMF.}
    \label{fig:examples}
\end{figure}


These difficult cases cannot be handled with separable NMF, because it assumes the data points to be linear combinations of vertices, so an ``interior vertex'' cannot be distinguished from another data point.
However, if we assume the \emph{sparsity} of the mixture matrix $H$, we may be able to identify these interior vertices.
To do so, we introduce a new model, extending the approach of SNPA using additional sparsity constraints.
We introduce in \cref{sec:ssnmf} a proper definition of this new problem, which we coin as sparse separable NMF (SSNMF). Before doing so, let use recall the literature on sparse NMF.

\subsection{Sparse NMF}

A vector or matrix is said to be \emph{sparse} when it has few non-zero entries.
Sparse NMF is one of the most popular variants of NMF, as it helps producing more interpretable factors.
In this model, we usually consider column-wise sparsity of the factor $H$, meaning that a data point is expressed as the combination of only a few atoms.
For example, in multispectral unmixing, the column-wise sparsity of $H$ means that a pixel is composed of fewer materials than the total number of materials present in the image.
When sparsity is an a priori knowledge on the structure of the data, encouraging sparsity while computing NMF is likely to reduce noise and produce better results.

Sparse NMF is usually solved by extending standard NMF algorithms with a regularization such as the $\ell_1$ penalty~\cite{hoyer2002non,kim_sparse_2007}, or constraints on some sparsity measure, like the one introduced in~\cite{hoyer2004non}.
Recently, exact $k$-sparse methods based on the $\ell_0$-``norm'' have been used for NMF, using a brute-force approach~\cite{cohen_nonnegative_2019-1}, or a dedicated branch-and-bound algorithm~\cite{nadisicexact}.
They allow the explicit definition of a maximum number (usually noted $k$) of non-zero entries per column of $H$.
These approaches leverage the fact that, in most NMF problems, the factorization rank $r$ is small, hence it is reasonable to solve the $k$-sparse NNLS subproblems exactly.

\subsection{Contributions and Outline}


In this work, we study the SSNMF model from a theoretical and a  pratical point of view. 
Our contributions can be summarized as follows:
\begin{itemize}
    \item In \cref{sec:ssnmf}, we introduce the SSNMF model. We prove that, unlike separable NMF, SSNMF is NP-complete.
    
    \item In \cref{sec:algo}, we propose an algorithm to tackle SSNMF, based on SNPA and an exact sparse NNLS solver.
    
    \item In \cref{sec:analysis}, we prove that our algorithm is correct under reasonable assumptions, in the noiseless case. 
    
    \item In \cref{sec:xp}, experiments on both synthetic and real-world data sets illustrate the relevance and efficiency of our algorithm. 
\end{itemize}


\section{Sparse Separable NMF}
\label{sec:ssnmf}

We explained in the previous section why separable NMF does not allow for the identification of ``interior vertices'', as they are nonnegative linear  combinations of other vertices.
However, if we assume a certain column-wise sparsity on the coefficient matrix $H$, they may become identifiable.
For instance, the vertex~$W_4$ of \cref{fig:intvertex} can be expressed as a combination of the three exterior vertices ($W_1$, $W_2$, and $W_3$), but not as a combination of any two of these vertices. 
Moreover, some data points cannot be explained using only pairs of exterior vertices, while they can be if we also select the interior vertex~$W_4$. 


\subsection{Problem Statement and Complexity}
\label{subsec:pbm}

Let us denote $\|x\|_0$ the number of non-zero entries of the vector $x$. 

\begin{definition}\label{def:sparsesep}
A matrix $M$ is $k$-sparse $r$-separable if there exists a subset of $r$ columns of $M$, indexed by $\mathcal{J}$, and a nonnegative matrix $H \geq 0$ with $\|H(:,j)\|_0 \leq k$ for all $j$ such that $M = M(:,\mathcal{J}) H$.  
\end{definition}
\Cref{def:sparsesep} corresponds to \cref{def:separability} with the additional constraint that $H$ has $k$-sparse columns, that is, columns with at most $k$ non-zero entries. 
A natural assumption to ensure that we can identify $W$ (that is, find the set $\mathcal{J}$), is that the columns of $W$ are not $k$-sparse combinations of any other columns of $W$; see \cref{sec:analysis} for the details.

\begin{problem}[SSNMF] 
Given a $k$-sparse $r$-separable matrix $M$, find $W = M(:,\mathcal{J})$ with $|\mathcal{J}| = r$ and a column-wise $k$-sparse matrix $H \geq 0$  such that $M = W H$. 
\end{problem}


As opposed to separable NMF, given $\mathcal{J}$, computing $H$ is not straightforward. It requires to solve the following $\ell_0$-constrained optimization problem 
\begin{equation}\label{eq:sparseproj}
    H^* = \argmin_{H \geq 0} f(M - M(:,\mathcal{J}) H) \text{ such that }  \|H(:,j)\|_0 \leq k \text{ for all } j. 
\end{equation}
Because of the combinatorial nature of the $\ell_0$-``norm'', this $k$-sparse projection is a difficult subproblem with $\binom{r}{k}$ possible solutions, which is known to be NP-hard~\cite{natarajan1995sparse}. 
In particular, a brute-force approach could tackle this problem by solving $\mathcal{O}(r^k)$ NNLS problems. 
However, this combinatorial subproblem can be solved exactly and at a reasonable cost by dedicated branch-and-bound algorithms, such as  \arbo{}~\cite{nadisicexact}, given that $r$ is sufficiently small, which is typically the case in practice.
Even when $k$ is fixed, the following result shows that no provably correct algorithm exists for solving SSNMF in polynomial time (unless P=NP):
\begin{theorem} \label{th:nphardssnmf}
SSNMF is NP-complete for any fixed $k \geq 2$. 
\end{theorem}

\begin{proof} 
The proof is given in Appendix~\ref{sec:nphard}. 
Note that the case $k = 1$ is trivial since each data point is a multiple of a column of $W$. 
\end{proof}

However, in Section \ref{sec:analysis}, we show that under a reasonable assumption, SSNMF can be solved in polynomial time when $k$ is fixed.

\subsection{Related Work} \label{sec:relwork}

To the best of our knowledge, the only work presenting an approach to tackle SSNMF is the one by Sun and Xin (2011)~\cite{sun2011underdetermined} --- and it does so only partially.
It studies the blind source separation of nonnegative data in the underdetermined case.
The problem tackled is equivalent to NMF in the case $m < r$.
The 
assumptions 
used in this work are similar to ours, that is, separability and sparsity. 
However, the setup considered is less general than SSNMF because the sparsity assumption (on each column of $H$) is limited to $k = m - 1$, while the only case considered theoretically is the case $r = m+1$ with only one interior vertex. 

The proposed algorithm first extracts the exterior vertices using the method LP-BSS from~\cite{naanaa2005blind}, and then identifies the interior vertex using a brute-force geometric method. 
More precisely, 
they select an interior point, and check whether at least two of the $m-1$ hyperplanes generated by this vertex with $m-2$ of the extracted exterior vertices contain other data points. 
If it is the case, then they conclude that the selected point is an interior vertex, otherwise they select another interior point. 
For example, when $m = 3$,  this method consists in constructing the  segments between the selected interior point and all the exterior vertices.  If two of these segments contain at least one data point, then the method stops and the selected interior point is chosen as the interior vertex. Looking at Figure~\ref{fig:intvertex}, the only interior point for which two segments joining this point and an exterior vertex   contain data points is $W_4$. 
Note that, to be guaranteed to work, 
this method requires at least two hyperplanes containing the interior vertex and $m-2$ exterior vertices to contain data points. This will not be a requirement in our method.  


\section{Proposed Algorithm: \brassens{}} 
\label{sec:algo}

In the following, we assume that the input matrix $M$ is $k$-sparse $r$-separable.
Our algorithm, called \brassens{}\footnote{It stands for \brassens{} Relies on Assumptions of Separability and Sparsity for Elegant NMF Solving.}, is presented formally in \cref{alg:ssnmf}.
\begin{algorithm}[tp]
  \caption{\brassens}\label{alg:ssnmf}

  \KwIn{A $k$-sparse-near-separable matrix ${M} \in \mathbb{R}^{m \times n}$, and the desired sparsity level $k$.}

  \KwOut{A set of $r$ indices $\mathcal{J}$, and a matrix $H  \in \mathbb{R}^{r \times n}_{+}$, such that $\|H(:,j)\|_0 \leq k$ for all $j$ and ${M} = \tilde{M}(:,\mathcal{J}) H$.}
  
  \vspace{0.2cm}
  \DontPrintSemicolon

  $\mathcal{J} = \text{SNPA}({M},\infty)$\label{alg:dosnpa}\\
  $\mathcal{J}' = \text{kSSNPA}({M},\infty,\mathcal{J})$\label{alg:dokssnpa}\\
  \ForEach{$j \in \mathcal{J}'$\label{alg:loop}}
  {
    \If{$\min\limits_{||h||_0 \leq k, h \geq 0} f( M(:,j) - M(:,\mathcal{J}' \setminus {j})  h ) > 0$\label{alg:if}}
    {
      $\mathcal{J} = \mathcal{J} \cup \{ j \}$\label{alg:addtoJ}
    }
  }
  $ H = \text{arborescent}(M,M(:,\mathcal{J}),k) $\label{alg:computeh}
\end{algorithm}

On \cref{alg:dosnpa} we apply the original SNPA to select the exterior vertices; it is computationally cheap and ensures that these vertices are properly identified. The symbol $\infty$ means that SNPA stops only when the residual error is zero. For the noisy case, we replace the condition $R \neq 0$ by $R > \delta$, where $\delta$ is a user-provided noise-tolerance threshold.

Then, we adapt SNPA to impose a $k$-sparsity constraint on $H$: 
  the projection step (\cref{alg:snpa-proj} of \cref{alg:snpa}) is replaced by a $k$-sparse projection step that imposes the columns of $H$ to be $k$-sparse by solving~\eqref{eq:sparseproj}. 
  We call \emph{kSSNPA} this modified version of SNPA. 
  Note that, if $k = r$, kSSNPA reduces to SNPA.  

On \cref{alg:dokssnpa} we apply kSSNPA to select candidate interior vertices. 
We provide it with the set $\mathcal{J}$ of exterior vertices so that they do not need to be identified again.
kSSNPA extracts columns of $M$ as long as the norm of the residual \mbox{$\| M-M(:,\mathcal{J}') H \|_F$} is larger than zero.
At this point, all vertices have been identified: the exterior vertices have been identified by SNPA, while the interior vertices have been identified by kSSNPA because we will assume that they are not $k$-sparse combinations of any other data points; see  Section~\ref{sec:analysis} for the details. 
Hence, the error will be equal to zero if and only if all vertices have been identified.
However, some selected interior points may not be interior vertices, because the selection step of kSSNPA chooses the point that is furthest away from the \emph{$k$-sparse hull} of the selected points, that is, the union of the convex hulls of the subsets of $k$ already selected points.
For example, in \cref{fig:intvertex}, if $W_1$, $W_2$, and $W_3$ are selected, the $k$-sparse hull is composed of the 3 segments  $[W_1,W_2]$, $[W_2,W_3]$, and $[W_3,W_1]$.
In this case, although only point~$W_4$ is a interior vertex, point~$M_1$ is selected before point~$W_4$, because it is located further away from the $k$-sparse hull.

On \crefrange{alg:loop}{alg:addtoJ}, we apply a postprocessing to the selected points by checking whether they are $k$-sparse combinations of other selected points; this is a $k$-sparse NNLS problem solved with \arbo{}~\cite{nadisicexact}.
If they are, then they cannot be vertices and they are discarded, such as point~$M_1$ in Figure~\ref{fig:intvertex} which belongs to the segment $[W_1,W_4]$.

Note that this ``postprocessing'' could be applied directly to the whole data set by selecting data points as the columns of $W$ if they are not $k$-sparse combinations of other data points. 
However, this is not reasonable in practice, as it is equivalent to solving $n$ times a $k$-sparse NNLS subproblem in $n-1$ variables. 
The kSSNPA step can thus be interpreted as a safe screening technique, similarly as done in~\cite{el2010safe} for example, in order to reduce the number of candidate atoms from all the columns of $M$ to a subset  $\mathcal{J}'$ of columns. In practice, we have observed that kSSNPA is very effective at identifying good candidates points; see Section~\ref{sec:xp-synth}.

\section{Analysis of \brassens{}}
\label{sec:analysis}

In this section, we first discuss the assumptions that guarantee \brassens{} to recover $W$ given the $k$-sparse $r$-separable matrix $M$, 
and then discuss the computational complexity of \brassens{}. 


\subsection{Correctness}
\label{subsec:correct}

In this section, we show that, given a $k$-sparse $r$-separable matrix $M$, the \brassens{} algorithm provably solves SSNMF, that is, it is able to recover the correct set of indices $\mathcal{J}$ such that ${W = M(:,\mathcal{J})}$, under a reasonable assumption.  

Clearly, a necessary assumption for \brassens{} to be able to solve SSNMF is that no column of $W$ is a $k$-sparse nonnegative linear combinations of other columns of $W$, otherwise kSSNPA might set that column of $W$ to zero, hence might not be able to extract it. 
\begin{assumption}\label{as:notcombiw}
No column of $W$ is a nonnegative linear combination of $k$ other columns of $W$.
\end{assumption}

Interestingly, unlike the standard separable case (that is, $k=r$), and although it is necessary in our approach with \brassens{},
Assumption~\ref{as:notcombiw} is not necessary in general to be able to uniquely recover $W$. 
Take for example the situation of \cref{fig:3intvertaligned}, with three aligned points in the interior of a triangle, so that $r=6$, $m=3$ and $k=2$. 
The middle point of these three aligned points is a 2-sparse combination of the other two, by construction.
If there are data points on each segment joining these three interior points and the exterior vertices, the only solution to SSNMF with $r=6$ is the one selecting these three aligned points. 
However, Assumption~\ref{as:notcombiw} is a reasonable assumption for SSNMF. 

\begin{figure}[t]
    \begin{minipage}{0.47\linewidth}
      \centering 
\tikzstyle{datapoint}=[shape=circle, draw=black, fill=mygreen, inner sep=0pt,minimum size=6pt]
\tikzstyle{extvertex}=[regular polygon, regular polygon sides=3, draw=black, fill=mypurple, inner sep=0pt,minimum size=8pt]
\tikzstyle{intvertex}=[regular polygon, regular polygon sides=4, draw=black, fill=myorange, inner sep=0pt,minimum size=8pt]

\begin{tikzpicture}[scale=0.25]
		\node [style=extvertex] (0) at (-1.25, 8) {};
		\node [style=extvertex] (1) at (-7.25, -2.75) {};
		\node [style=extvertex] (2) at (5.5, -3.5) {};
		\node [style=intvertex, label=above:{$W_6$}] (3) at (1.75, 0.75) {};
		\node [style=intvertex, label=above:{$W_4$}] (4) at (-4, 0.75) {};
		\node [style=intvertex, label=above:{$W_5$}] (5) at (-1, 0.75) {};
		\draw (1) --  (0);
		\draw (0) --  (2);
		\draw (2) --  (1);
		\draw (1) --  (3);
		\draw (3) --  (0);
		\draw (3) --  (4);
		\draw (4) --  (2);
		\draw (1) --  (4);
		\draw (4) --  (0);
		\draw (3) --  (2);
		\draw (5) --  (1);
		\draw (5) --  (0);
		\draw (5) --  (2);
		\node [style=intvertex] (5) at (-1, 0.75) {};
\end{tikzpicture}
      \caption{There are three interior vertices. One of them ($W_5$) is a combination of the others two.}\label{fig:3intvertaligned}
   \end{minipage}\hfill
    \centering
    \begin{minipage}{0.47\linewidth}
      \centering 
\tikzstyle{datapoint}=[shape=circle, draw=black, fill=mygreen, inner sep=0pt,minimum size=6pt]
\tikzstyle{extvertex}=[regular polygon, regular polygon sides=3, draw=black, fill=mypurple, inner sep=0pt,minimum size=8pt]
\tikzstyle{intvertex}=[regular polygon, regular polygon sides=4, draw=black, fill=myorange, inner sep=0pt,minimum size=8pt]

\begin{tikzpicture}[scale=0.25]
		\node [style=extvertex] (0) at (-1.25, 8) {};
		\node [style=extvertex] (1) at (-7.25, -2.75) {};
		\node [style=extvertex] (2) at (5.5, -3.5) {};
		\node [style=intvertex] (3) at (-2.5, 3.25) {};
		\node [style=intvertex, label=below:{$W_4$}] (4) at (-0.25, -1) {};
		\node [style=datapoint, label={$M_1$}] (5) at (2.5, -1) {};
		\node [style=datapoint, label={[label distance=1pt]150:$M_2$}] (6) at (-5.75, -1) {};
		\draw (1) -- (0);
		\draw (0) -- (2);
		\draw (2) -- (1);
		\draw (1) -- (3);
		\draw (3) -- (0);
		\draw (3) -- (4);
		\draw (4) -- (2);
		\draw (1) -- (4);
		\draw (4) -- (0);
		\draw (3) -- (2);
		\draw [dotted] (5) -- (6);
		\node [style=datapoint] at (2.5, -1) {};
		\node [style=datapoint] at (-5.75, -1) {};
\end{tikzpicture}
    \caption{There are two interior vertices. One of them ($W_4$) is a 2-sparse combination of data points ($M_1$ and $M_2$).}\label{fig:badex}
   \end{minipage}
\end{figure}

Unfortunately, Assumption~\ref{as:notcombiw} is not sufficient for \brassens{} to provably recover $W$. In fact, we need the following stronger assumption. 
\begin{assumption}\label{as:notcombim}
No column of $W$ is a nonnegative linear combination of $k$ other columns of $M$.
\end{assumption}
This assumption guarantees that a situation such as the one shown on  Figure~\ref{fig:badex} where one of the columns of $W$ is a 2-sparse combination of two data points is not possible. In fact, in that case, if \brassens{} picks these two data points before the interior vertex $W_4$ in between them, it will not be able to identify $W_4$ as it is set to zero within the projection step of kSSNPA. 

Interestingly, in the standard separable case, that is, $k=r$, the two assumptions above coincide; this is the condition under which SNPA is guaranteed to work. 
Although Assumption~\ref{as:notcombim} may appear much stronger than Assumption~\ref{as:notcombiw}, they are actually generically equivalent given that the entries of the columns of $H$ are generated randomly (that is, non-zero entries are picked at random and follow some {continuous distribution}). 
For instance, for $m=3$ and $k=2$, it means that no vertex is on a segment joining two data points. If the data points are generated randomly on the segments generated by any two columns of $W$, the probability for the segment defined by two such data points to contain a column of $W$ is zero. In fact, segments define a set of measure zero in the unit simplex.

We can now provide a recovery result for \brassens{}.
\begin{theorem} \label{th:recovery}
Let $M = W H$ with $W = M(:,\mathcal{J})$ be a $k$-sparse $r$-separable matrix so that $|\mathcal{J}| = r$; see \cref{def:sparsesep}. 
We have that 
\begin{itemize}
    \item If $W$ satisfies \cref{as:notcombim}, 
    then the factor $W$ with $r$ columns in SSNMF is unique (up to permutation and scaling)  and 
    \brassens{} recovers it.   
    
    \item If $W$ satisfies \cref{as:notcombiw}, the entries of $H$ are generated at random (more precisely, the position of the non-zero entries are picked at random, while their values follows a continuous ditribution) and $k < \rank(M)$, then,  with probability one,  the factor $W$ with $r$ columns in SSNMF is unique (up to permutation and scaling)  and 
    \brassens{} recovers it.     
\end{itemize}
\end{theorem}
\begin{proof} 
Uniqueness of $W$ in SSNMF under \cref{as:notcombim} is straightforward: since the columns of $W$ are not $k$-sparse combinations of other columns of $M$, they have to be selected in the index set $\mathcal{J}$. Otherwise, since columns of $W$ are among the columns of $M$, it would not possible to reconstruct $M$ exactly using $k$-sparse combinations of $M(:,\mathcal{J})$. Then, since all other columns are $k$-sparse combinations of the $r$ columns of $W$ (by assumption), no other columns needs to be added to $\mathcal{J}$ which satisfies $|\mathcal{J}|=r$.  

Let us show that, 
under \cref{as:notcombim}, \brassens{} recovers the correct set of indices $\mathcal{J}$. 
kSSNPA can only stop when all columns of $W$ have been identified. 
In fact, kSSNPA stops when the reconstruction error is zero, 
while, under \cref{as:notcombim}, this is possible only when all columns of $W$ are selected (for the same reason as above). 
Then, the postprocessing will be able to identify, among all selected columns, the columns of $W$, because they will be the only ones that are not $k$-sparse combinations of other selected columns. 

The second part of the proof follows from standard probabilistic results: since $k < \rank(M)$, the combination of $k$ data points generates a subspace of dimension smaller than that of $\col(M)$. Hence, generating data points at random is equivalent to generating such subspaces at random. Since these subspaces form a space of measure zero in $\col(M)$, the probability for these subspaces to contain a column of $W$ is zero, which implies that \cref{as:notcombim} is satisfied with probability one. 
\end{proof} 

\subsection{Computational Cost}

Let us derive an upper bound on the computational cost of \brassens{}. 
First, recall that solving an NNLS problem up to any precision can be done in polynomial time. For simplicity and because we focus on the non-polynomial part of \brassens{}, we denote $\bar{\mathcal{O}}(1)$ the complexity of solving an NNLS problem.  
In the worst case, kSSNPA will extract all columns of $M$. In each of the $n$ iterations of kSSNPA, the problem~\eqref{eq:sparseproj} needs to be solved. When $|\mathcal{J}|=\mathcal{O}(n)$, this requires to solve $n$ times (one for each column of $M$) a $k$-sparse least squares problem in $|\mathcal{J}|=\mathcal{O}(n)$ variables. The latter requires in the worst case $\mathcal{O}(n^k)$ operations by trying all possible index sets; see the discussion after~\eqref{eq:sparseproj}. In total, kSSNPA will therefore run in the worst case in time $\bar{\mathcal{O}}(n^{k+2})$. 

Therefore, when $k$ is fixed (meaning that $k$ is considered as a fixed constant) and under \cref{as:notcombim}, \brassens{} can solve SSNMF in polynomial time. 
Note that this is not in contradiction with our NP-completeness results when $k$ is fixed (Theorem~\ref{th:nphardssnmf}) because our NP-completeness proof does not rely on \cref{as:notcombim}. 

In summary, to make SSNMF hard, we need either $k$ to be part of the input, or the columns of $W$ to be themselves $k$-sparse combinations of other columns of $W$. 


\section{Experiments}
\label{sec:xp}

The code and data are available online\footnote{\url{https://gitlab.com/nnadisic/ssnmf}}.
All experiments have been performed on a personal computer with an i5
processor, with a clock frequency of 2.30GHz.
All algorithms are single-threaded.
All are implemented in Matlab, except the sparse NNLS solver \arbo{}, which is implemented in C++ with a Matlab MEX interface. 

As far as we know, no algorithm other than \brassens{} can tackle SSNMF with more than one interior point (see Section~\ref{sec:relwork}) hence comparisons with existing works are unfortunately limited.
For example, separable NMF algorithms can only identify the exterior vertices; see Section~\ref{sec:intro}.
However, we will compare \brassens{} to SNPA on a real multispectral image in Section~\ref{sec:xp-real}, to show the advantages of the SSNMF model over separable NMF.
In Section~\ref{sec:xp-synth}, we illustrate the correctness and efficiency of \brassens{} on synthetic data sets.

\subsection{Synthetic Data Sets}\label{sec:xp-synth}

In this section, we illustrate the behaviour of \brassens{} in different experimental setups. 
The generation of a synthetic data set is done as follows: for a given number of dimensions $m$, number of vertices $r$, number of data points $n$, and data sparsity $k$, we generate matrices $W \in \mathbb{R}_+^{m \times r}$ such that the last $r-m$ columns of $W$ are linear combinations of the first $m$ columns, and $H \in \mathbb{R}_+^{r \times n}$ such that $H = [ I_r , H' ]$ and $\|H(:,j)\|_0 \leq k$ for all $j$.
We use the uniform distribution in the interval [0,1] to generate random numbers (columns of $W$ and columns of $H$), and then normalize the columns of $W$ and $H$ to have unit $\ell_1$ norm.
We then compute $M = W H$.
This way, the matrix $M$ is $k$-sparse $r$-separable, with $r$ vertices, of which $r-m$ are interior vertices (in fact, the first $m$ columns of $W$ are linearly independent with probability one as they are generated randomly). 
We then run \brassens{} on $M$, with the parameter $k$, and no noise-tolerance.
For a given setup, we perform 30 rounds of generation and solving, and we measure the median of the running time and the median of the number of candidates extracted by kSSNPA.
This number of candidates corresponds to $| \mathcal{J}' |$ in \cref{alg:ssnmf}, that is, the number of interior points selected by kSSNPA as potential interior vertices.
Note that a larger number of candidates only results in an increased computation time, and does not change the output of the algorithm which is guaranteed to extract all vertices (\cref{th:recovery}).

\cref{fig:xpnvar} shows the behaviour of \brassens{} when $n$ varies, with fixed $m=3$, $k=2$, and $r=5$.
To the best of our knowledge, this case is not handled by any other algorithm in the literature.
Both the number of candidates and the run time grow slower than linear.
The irregularities in the plot are due to the high variance between runs.
Indeed, if vertices are generated in a way that some segments between vertices are very close to each other, \brassens{} typically selects more candidates before identifying all columns of $W$.  
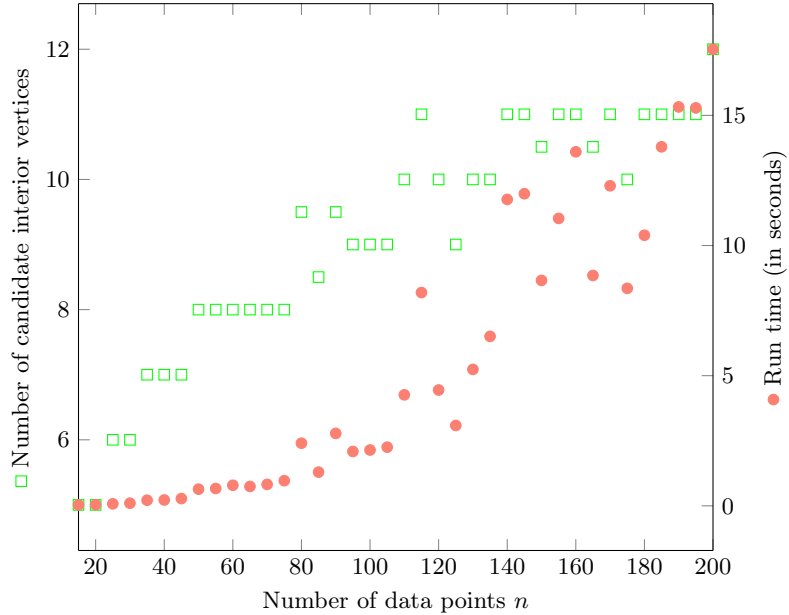
\begin{figure}[ht!] 
    \centering
    \begin{tikzpicture}
    \begin{axis}[
      scale only axis,
      xmin=15,
      xmax=200,
      axis y line*=left, 
      xlabel={Number of data points $n$},
      ylabel={\ref{cand} Number of candidate interior vertices}]
      \addplot[only marks, color=green, mark=square] table [x index=0,y index=1] {img/xp-nvar.txt}; \label{cand}
    \end{axis}

    \begin{axis}[
      scale only axis,
      xmin=15,
      xmax=200,
      axis y line*=right,
      axis x line=none,
      ylabel={\ref{runtime} Run time (in seconds)}]
       \addplot[only marks, color=myorange] table [x index=0,y index=2, color=green] {img/xp-nvar.txt}; \label{runtime}
    \end{axis}
\end{tikzpicture}
    \caption{Results for \brassens{} on synthetic data sets for different values of $n$, with fixed $m=3$, $k=2$, and $r=5$ with 3 exterior and 2 interior vertices. The values showed are the medians over 30 experiments.}\label{fig:xpnvar}
\end{figure}

In \cref{tab:synth2} we compare the performance of \brassens{} for several sets of parameters. 
The number of candidates grows relatively slowly as the dimensions $(m,n)$ of the problem increase, showing the efficiency of the screening performed by kSSNPA. 
However, the run time grows rather fast when the dimensions $(m,n)$ grow.  
This is because, not only the number of NNLS subproblems to solve increase, but also their size. 
\setlength{\tabcolsep}{6pt}
\begin{table}[ht]
\centering
\caption{Results for \brassens{} on synthetic data sets (median over 30 experiments).} 
\label{tab:synth2}
\begin{tabular}{c|c|c|c|c|c}
m & n  & r  & k & Number of candidates & Run time in seconds \\ \hline
3 & 25 & 5  & 2 & 5.5                  & 0.26                \\
4 & 30 & 6  & 3 & 8.5                  & 3.30                \\
5 & 35 & 7  & 4 & 9.5                  & 38.71               \\
6 & 40 & 8  & 5 & 13                   & 395.88              \\
\end{tabular}
\end{table}
In all cases, as guaranteed by Theorem~\ref{th:recovery}, \brassens{} was able to correctly identify the columns of $W$. Again, as far as we know, no existing algorithms in the literature can perform this task. 

To summarize, our synthetic experiments show the efficiency of the screening done by kSSNPA, and the capacity of \brassens{} to handle medium-scale data sets.

\subsection{Blind Multispectral Unmixing}\label{sec:xp-real}

A multispectral image is an image composed of various wavelength ranges, called \emph{spectral bands}, where every pixel is described by its \emph{spectral signature}.
This signature is a vector representing the amount of energy measured for this pixel in every considered spectral band.
Multispectral images usually have a small number of bands (between 3 and 15).
These bands can be included or not in the spectrum of visible light.
In the NMF model, if the columns of $M$ are the $n$ pixels of the image, then its rows represent the $m$ spectral bands.

The unmixing of a multispectral image consists in identifying the different materials present in that image.
When the spectral signatures of the materials present in the image are unknown, it is referred to as \emph{blind unmixing}.
The use of NMF for blind unmixing of multispectral images relies on the linear mixing model, that is, the assumption that the spectral signature of a pixel is the linear combination of the spectral signatures of the material present in this pixel.
This corresponds exactly to the NMF model, which is therefore able to identify both the materials present in the image ($W$) and the proportions/abundances of materials present in every pixel ($H$); see~\cite{bioucas2012hyperspectral,ma2013signal} for more details. 

Let us apply \brassens{} to the unmixing of the well-known Urban satellite image~\cite{fyzhu2014hyperspectraldata}, composed of 
$309 \times 309$ pixels. The original cleaned image has 162 bands, but we only keep 3 bands, namely the bands 2, 80, and 133 -- these were obtained by selecting different bands with SPA applied on $M^T$ -- to obtain a data set of size $3 \times \num{94249}$. The question is: can we still recover materials by using only 3 bands? (The reason for this choice is that this data set is well known and the ground truth is available, which is not the case of most multispectral images with only 3 bands.)  
We first normalize all columns of $M$ so that they sum to one.
Then, we run \brassens{} with a sparsity constraint $k = 2$ (this means that we assume that a pixel can be composed of at most 2 materials, which is reasonable for this relatively high resolution image) and a noise-tolerance threshold of 4\%; this means that we stop SNPA and kSSNPA when 
$\|M - M(:,\mathcal{J}) H\|_F \leq 0.04 \|M\|_F$. 
\brassens{} extracts 5 columns of the input matrix.  
For comparison, we run SNPA with $r = 5$. 
Note that this setup corresponds to underdetermined blind unmixing, because $m = 3 < r = 5$. 
It would not be possible to tackle this problem using standard NMF algorithms (that would return a trivial solution such as $M = I_3 M$). 
It can be solved with SNPA, but SNPA cannot identify interior vertices. 

SNPA extracts the 5 vertices in \num{3.8} seconds.
\brassens{} extracts 5 vertices, including one interior vertex, in \num{33} seconds. 
The resulting abundance maps are showed in \cref{fig:mapbrassens}.
They correspond to the reshaped rows of $H$, hence they show which pixel contains which extracted material (they are more easily interpretable than the spectral signatures contained in the columns of $W$).  
The materials they contain are given in \cref{tab:unmixing}, using the ground truth from~\cite{zhu2017hyperspectral}.
We see that \brassens{} produces a better solution, as the materials present in the image are better separated: the first three abundance maps of \brassens{} are sparser and correspond to well-defined materials. The last two abundances maps of SNPA and of \brassens{} are similar but extracted in a different order. 
The running time of \brassens{} is reasonable, although ten times higher than SNPA.

\begin{figure}[ht!]
    \centering
    \includegraphics[width=\textwidth]{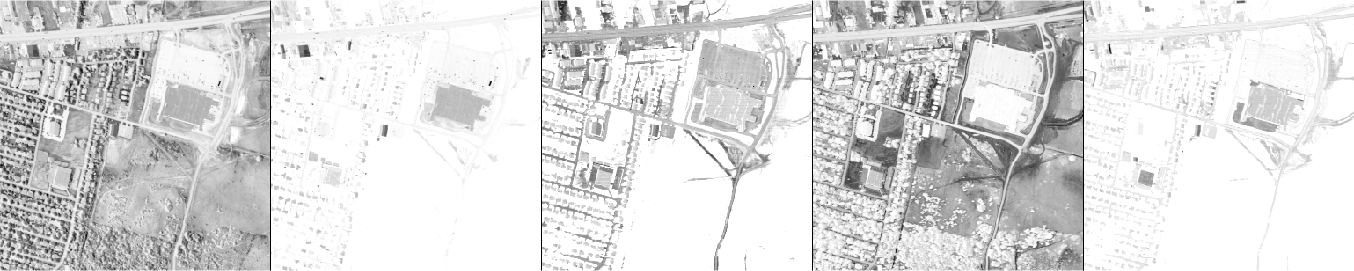}
    \vspace{0.1cm}\\ 
    \includegraphics[width=\textwidth]{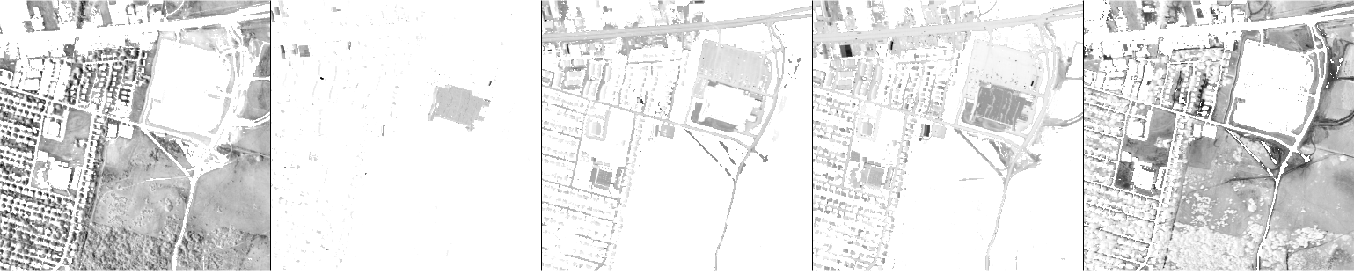} 
    \caption{Abundances maps of materials (that is, reshaped rows of $H$) extracted by SNPA (top) and by \brassens{} (bottom) in the Urban image with only 3 spectral bands.}\label{fig:mapbrassens}
\end{figure}

\begin{table}[ht!] 
\centering
\caption{Interpretation of the unimixing results from Figure~\ref{fig:mapbrassens}.}\label{tab:unmixing}
\begin{tabular}{c|c|c}
Image & Materials extracted by SNPA & Materials extracted by \brassens{} \\ \hline
1     & Grass + trees + roof tops   & Grass + trees                   \\
2     & Roof tops 1                 & Roof tops 1                     \\
3     & Dirt + road + roof tops     & Road                            \\
4     & Dirt + grass                & Roof tops 1 and 2 + road        \\
5     & Roof tops 1 + dirt + road   & Dirt + grass                   
\end{tabular}
\end{table}
\section{Conclusion}
\label{sec:conclu}

In this paper, we introduced SSNMF, a new variant of the NMF model combining the assumptions of separability and sparsity. 
We presented \brassens{}, an algorithm able to solve exactly SSNMF, based on SNPA and an exact sparse NNLS solver.
We showed its efficiency for various setups and in the successful unmixing of a multispectral image. 
The present work provides a new way to perform underdetermined blind source separation, under mild hypothesis, and a new way to regularize NMF.
It makes NMF identifiable even when atoms of $W$ are nonnegative linear combinations of other atoms (as long as these combinations have sufficiently many non-zero coefficients).   
Further work includes the theoretical analysis of the proposed model and algorithm in the presence of noise.

\subsubsection*{Acknowledgments.}

The authors are grateful to the reviewers, whose insightful comments helped improve the paper. 
NN and NG acknowledge the support by the European Research Council (ERC starting grant No 679515), and by the Fonds de la Recherche Scientifique - FNRS and the Fonds Wetenschappelijk Onderzoek - Vlanderen (FWO) under EOS project O005318F-RG47. 
\appendix

\section{Proof of Theorem~1: NP-Completeness of SSNMF} \label{sec:nphard}

The main purpose of this material is to provide the proof of Theorem~1.
More precisely, we prove the NP-completeness of SSNMF with $k=2$, which we denote 2-SSNMF.
The decision version of this problem is formally defined as follows.
\begin{problem}2-SSNMF \label{sparseNMF}

  \noindent
  Given: a natural number $r>0$ and a $2$-sparse $r$-separable matrix $M$
  
  \noindent
  Question: find a dictionary matrix $W = M(:,\mathcal{J})$ with $|\mathcal{J}| \leq r$ and a column-wise $2$-sparse matrix $H \geq 0$  such that $M = W H$. 

\end{problem}
NP-completeness of SSNMF for any $2 \leq k \leq r$ follows directly as it would allow to solve 2-SSNMF by simply adding artificial columns of $W$ (for example orthogonal to the ones used in the 2-sparse decomposition). 
In order to prove the NP-hardness of 2-SSNMF, we first demonstrate a polynomial time reduction from the well known NP-complete problem \textit{SET-COVER} (see Garey and Johnson (2002)\footnote{Garey, M.R., Johnson, D.S.: Computers and intractability, vol. 29 (2002)}) to 2-SSNMF.

\begin{problem}SET-COVER \label{setcover}

  \noindent
  Given: A finite set $S=\{1,...,n\}$, a collection $C=\{C_1,...,C_m\}$ of subsets of $S$ and a positive integer $K\leq m$.

  \noindent
  Question: Does $C'\subseteq C$ exist with $|C'|\leq K$ such that every element of $S$ belongs to at least one member of $C'$.
\end{problem}

From an instance $(S,C,K)$ of SET-COVER, let us construct an instance $(M,r)$ of 2-SSNMF in polynomial time.

\begin{itemize}
\item {The natural number $r>0$} is defined as
  {
    \[
      r = \sum_{i=1}^m|C_i| + 2 + K.
    \]
  }
\item {The matrix $M$} is the concatenation of three matrices $M_1$, $M_2$, $M_3$ such that $M=[M_1,M_2,M_3]$.
    \begin{itemize}
        \item For each subset $C_i$ of Problem \ref{setcover} with $i=1,...,m$, we have the data point $M_1(:,i)$ defined as follows:
        \[
        M_1(:,i) = \begin{pmatrix} 0 & -h_i\end{pmatrix}^T
        \]
        with $h_i = \frac{i}{m+1}$.
        Hence, $M_1$ is a $2$-by-$m$ matrix.
        \item For each element $j=1,...,n$ of the ground set $S$, we have the data point $M_2(:,j)$ defined as follows:
        \[
        M_2(:,j) = \begin{pmatrix} b_j & b_j^2\end{pmatrix}^T
        \]
        with $b_j = \frac{1}{m+1+aj}$ and $a = \frac{1}{n}$.
        These points belong to the curve $y=x^2$.
        $M_2$ is a $2$-by-$n$ matrix.
        \item When the $j$th element of the ground set $S$ is a member of the $i$th subset $C_i$, we add a data point in $M_3$ as follows:
        \[
        M_3(:,l) = \begin{pmatrix} \frac{h_i}{b_j} &\frac{h_i^2}{b_j^2}\end{pmatrix}^T
        \]
        with $h_i$ and $b_j$ as previously defined.
        Moreover, we add two more columns to $M_3$ for the data points $(0,0)$ and $(0,-1)$.
        Hence, $M_3$ is a $2$-by-$(2+\sum_{i=1}^m|C_i|)$ matrix.
        Note that the intersection of the curve $y=x^2$ with the linear equation connecting $(0,-h_i)$ and $(b_j,b_j^2)$ is precisely the point $(\frac{h_i}{b_j},\frac{h_i^2}{b_j^2})$.
        We show in Lemma \ref{M3alldifferent} that all these points never overlap.
        It implies that a straight line between the $i$th point of $M_1$ and a point in $M_3$ is passing through the $jth$ point of $M_2$ if and only if $j$ is in the subset $C_i$.
    \end{itemize}

\end{itemize}

\noindent
\begin{lemma} \label{M3alldifferent}
All the columns of $M_3$ are different.
\end{lemma}
\begin{proof}

\noindent
Suppose it is not the case and that for $(i,j)$ and $(i',j')$ with $i\neq i'$ and $j\neq j'$, we have $\frac{h_i}{b_j}=\frac{h_{i'}}{b_{j'}}$, which means that, after rearrangement
\begin{equation}
  \frac{i}{i'}=\frac{m+1+aj'}{m+1+aj}. \label{equalityneververified}
\end{equation}

\noindent
For $j,j'=1,...,n$, the right-hand side of \eqref{equalityneververified} varies as follows:
\[
  1 - a\left(\frac{n-1}{m+1+an}\right) \leq \frac{m+1+aj}{m+1+aj'} \leq 1 + a\left(\frac{n-1}{m+1+a}\right),
\]
and when $a=\frac{1}{n}$, we have $a\left(\frac{n-1}{m+1+an}\right) < a\left(\frac{n-1}{m+1+a}\right)<\frac{1}{m+1}$, which means that the variation of the right-hand side around $1$ is as follows
\[
  1 - \frac{1}{m+1} < \frac{m+1+aj}{m+1+aj'} < 1 + \frac{1}{m+1}.
\]

\noindent
For $i,i'=1,...,m+1$, the closest value to $1$ of $\frac{i}{i'}$ is $\frac{m}{m+1}$, that is $1-\frac{1}{m+1}$.
Therefore, the choice of $a=\frac{1}{n}$ prevents the right-hand side  of (\ref{equalityneververified}) to be equal to its left-hand side.
It results that all the values $\frac{h_i}{b_j}$ are different for $i=1,...,m+1$ and $j=1,...,n$.
\end{proof}

\begin{lemma} \label{lemmaconvexhullM}
All the columns of $M$ are situated inside the convex hull of the columns of $M_3$.
\end{lemma}
\begin{proof}
Except for $(0,-1)$, the columns of $M_3$ are located on the moment curve $y=x^2$.
It results that these points are the vertices of a convex polygon, known under the name of \textit{cyclic polytope}.
The intersection of the $y$-axis and the line connecting any two points of the set $\{(x,y)|y=x^2, x\geq 1\}$ is located strictly below the point $(0,-1)$.
Following the definitions of $h_i$ and $b_j$, we have $\frac{h_i}{b_j}\geq 1$ for any $i=1,...,m$ and $j=1,...,n$, which means that even with the addition of $(0,-1)$, the points of $M_3$ still form a convex polygon.
It is then easy to check that the points of $M_1$ and $M_2$ are inside the convex hull of $M_3$.
\end{proof}

\begin{lemma}\label{yesyes}
  The \textit{2-SSNMF} instance is a yes-instance if and only if the \textit{SET-COVER} instance is a yes-instance.
\end{lemma}
\noindent
\begin{proof}

\noindent
\textbf{The \textit{if} part.} Suppose we have an optimal cover $C'\subseteq C$ of the \textit{SET-COVER} instance with $|C'|\leq K$.
From this solution, we build a solution to the \textit{2-SSNMF} instance as follows:
\begin{itemize}
    \item For the dictionary matrix $W$, we concatenate $M_3$ and the columns of $M_1$ corresponding to the subsets in $C'$.
By this way, the number of columns of $W$ is less or equal than $r=\sum_{i=1}^m|C_i| + 2 + K$.
    \item With $M_3$ being in the dictionary, it is easy to construct the columns of $H$ corresponding to $[M_1,M_3]$ in $M$: it is trivial for $M_3$ and, for $M_1$, the two nonnegative entries of a column of $H$ are the two coefficients of the convex combination of $(0,0)$, $(0,-1)$.
    Moreover, since $W$ also contains the $K$ columns of $M_1$ corresponding to the cover $C'$, every column coming from $M_2$ in $M$ can be expressed as the convex combination of exactly two columns of $W$ (see the reduction above).
    By this way, we have $H\geq 0$, a column-wise $2$-sparse matrix, such that $M=WH$.
\end{itemize}

\noindent
\textbf{The \textit{only if} part.}
Suppose that we have a solution $(W,H)$ of the \textit{2-SSNMF} instance such that $M=WH$, $W$ having at most $r$ columns and $H\geq 0$ being a column-wise $2$-sparse matrix.
From this factorization, we show how to extract a cover $C'$ made of at most $K$ subsets.
All the columns of $M_3$ are necessarily in $W$ since they are the vertices of a convex polygon (see Lemma \ref{lemmaconvexhullM}).
Since, by construction, no convex combination of two points in $M_3$ can reach the $n$ points in $M_2$, we must have $W=[M_3,W']\Pi$ with the columns of $W'$ coming either from $M_1$ or from $M_2$.
The number of columns of $W'$ is therefore $r-\left(\sum_{i=1}^m|C_i| + 2\right)=K$.
It remains to show how to construct a solution to the \textit{SET-COVER} instance from $W'$.
For every point in $M_2$, it is possible to find a point in $M_1$ and a point $M_3$ such that the three points are lined up (it is always possible to find such points since we suppose that every element of the ground set belongs to at least one subset in the \textit{SET-COVER} instance).
It means that we can replace all the columns coming from $M_2$ in $W'$ by columns of $M_1$ without increasing the size of $W'$.
In order to maintain the equality $M=WH$, it is easy to update the matrix $H$ accordingly while keeping it column-wise 2-sparse.
Finally, with the $K$ columns of $W'$ coming from $M_1$, we have identified a cover $C'$ composed of $K$ subsets for the \textit{SET-COVER} instance.
\end{proof}

\begin{proof}{Proof of Theorem~1.}
2-SSNMF is in NP since we can check in polynomial time that a given pair $(W,H)$ is a solution of a 2-SSNMF instance.
With the reduction from the SET-COVER problem presented above and Lemma \ref{yesyes}, we can conclude that 2-SSNMF is NP-hard.
\end{proof}

\paragraph{Illustration of the reduction.}
\setcounter{MaxMatrixCols}{20}
\noindent
From the \textit{SET-COVER} instance: $n=5$, $m=4$, $K=2$, $C_1=\{2,4\}$, $C_2=\{1,2,3\}$, $C_3=\{3,4\}$ and $C_4=\{4,5\}$, the reduction presented above leads to the following \textit{2-SSNMF} instance: $r=13$ and $M=[M_1,M_2,M_3]$ with
\[
    M_1 = \begin{pmatrix}
    0 & 0 & 0 & 0 \\
    -h_1 & -h_2 & -h_3 & -h_4
    \end{pmatrix}, 
    M_2 = \begin{pmatrix}
    b_1 & b_2 & b_3 & b_4 & b_5 \\
    b_1^2 & b_2^2 & b_3^2 & b_4^2 & b_5^2
    \end{pmatrix},\] 
    \[\text{ and }
    M_3 = \begin{pmatrix}
    \frac{h_1}{b_2} & \frac{h_1}{b_4} & \frac{h_2}{b_1} & \frac{h_2}{b_2} & \frac{h_2}{b_3} & \frac{h_3}{b_3} & \frac{h_3}{b_4} & \frac{h_4}{b_4} & \frac{h_4}{b_5}  & 0 & 0 \\
    \frac{h_1^2}{b_2^2} & \frac{h_1^2}{b_4^2} & \frac{h_2^2}{b_1^2} & \frac{h_2^2}{b_2^2} & \frac{h_2^2}{b_3^2} & \frac{h_3^2}{b_3^2} & \frac{h_3^2}{b_4^2} & \frac{h_4^2}{b_4^2} & \frac{h_4^2}{b_5^2}  & 0 & -1
    \end{pmatrix},
  \]  
where $h_i=\frac{i}{5}$ for $i=1,...4$ and $b_j=(5+\frac{j}{5})^{-1}$ for $j=1,...5$ (see 
Figure \ref{examplecomplexity}).

\noindent
A solution to the \textit{SET-COVER} instance is $C'=\{C_2,C_4\}$ and the corresponding \textit{2-SSNMF} solution is

\[
W = \begin{pmatrix}
    \frac{h_1}{b_2} & \frac{h_1}{b_4} & \frac{h_2}{b_1} & \frac{h_2}{b_2} & \frac{h_2}{b_3} & \frac{h_3}{b_3} & \frac{h_3}{b_4} & \frac{h_4}{b_4} & \frac{h_4}{b_5}  & 0 & 0 & 0 & 0\\
    \frac{h_1^2}{b_2^2} & \frac{h_1^2}{b_4^2} & \frac{h_2^2}{b_1^2} & \frac{h_2^2}{b_2^2} & \frac{h_2^2}{b_3^2} & \frac{h_3^2}{b_3^2} & \frac{h_3^2}{b_4^2} & \frac{h_4^2}{b_4^2} & \frac{h_4^2}{b_5^2}  & 0 & -1 & -h_2 & -h_4
    \end{pmatrix}, \text{ and }
\]
$H = $
{
\footnotesize
\[
H = \begin{pmatrix}
    0 & 0 & 0 & 0         & 0 & 0 & 0 & 0         & 0  & \\  
    0 & 0 & 0 & 0         & 0 & 0 & 0 & 0         & 0  & \\  
    0 & 0 & 0 & 0         & \beta_{2,1} & 0 & 0 & 0 & 0  & \\    
    0 & 0 & 0 & 0         & 0 & \beta_{2,2} & 0 & 0 & 0 & \\    
    0 & 0 & 0 & 0         & 0 & 0 & \beta_{2,3} & 0  & 0 & \\    
    0 & 0 & 0 & 0         & 0 & 0 & 0 & 0        & 0     & I_{11} \\    
    0 & 0 & 0 & 0         & 0 & 0 & 0 & 0        & 0      & \\  
    0 & 0 & 0 & 0         & 0 & 0 & 0 & \beta_{4,4}       & 0 & \\  
    0 & 0 & 0 & 0         & 0 & 0 & 0 & 0         & \beta_{4,5} & \\ 
    1-\alpha_1 & 1-\alpha_2 & 1-\alpha_3 & 1-\alpha_4 & 0 & 0 & 0 & 0         & 0 & \\ 
    \alpha_1 & \alpha_2 & \alpha_3 & \alpha_4 & 0 & 0 & 0 & 0         & 0  & \\ 
    0 & 0 & 0 & 0         & 1-\beta_{2,1} & 1-\beta_{2,2} & 1-\beta_{2,3} & 0   & 0 &  0 \hdots 0\\   
    0 & 0 & 0 & 0         & 0   & 0   & 0   & 1-\beta_{4,4} & 1-\beta_{4,5} & 0 \hdots 0\\ 
    \end{pmatrix},
\]
}
for which we have $M=WH$ when $\alpha_i=h_i$, $\beta_{i,j}=\frac{b_j^2}{h_i}$.

\begin{figure}
\centering
    \def\ra{8}
    \def\rx{2}
    \def\rb{0}
    \def\m{4}
    \def\n{5}
    \def\a{1/\n}
    \def\XMAX{(\m*(\m+1+\a*\n))/(\m+1)}
    \def\IXMAX{5}
    \def\YMAX{\XMAX*\XMAX}
    \def\IYMAX{\IXMAX*\IXMAX}
    \def\INTENSITENOIR{20}
    \def\INTENSITEROUGE{80}
    \begin{tikzpicture}[scale=2.5,cap=round,spy using outlines={
        rectangle,
        magnification=4,
        width=8cm,
        height=8cm,
        connect spies,
      }]
      \tikzstyle{axes}=[]
      \tikzstyle{important line}=[ultra thick]
      \tikzstyle{information text}=[rounded corners,fill=red!10,inner sep=1ex]

      \begin{scope}[style=axes]
        \draw[black!\INTENSITENOIR] (0,0) -- ({(\IXMAX)/(\rx)},0) node[right] {};
        \draw[black!\INTENSITENOIR] (0,-1.1/\ra) -- (0,{(\IYMAX)/(\ra)}) node[above] {};
        \foreach \x/\xtext in {1,...,\IXMAX}
        \draw[black!\INTENSITENOIR]({(\x+\rb)/(\rx)},1pt) -- ({(\x+\rb)/(\rx)},-1pt) node[below,fill=white]
        {$\xtext$};
        \foreach \x/\xtext in {0}
        \draw[black!\INTENSITENOIR](\x+\rb,1pt) -- (\x+\rb,-1pt) node[left,fill=white]
        {$\xtext$};
        \foreach \y/\ytext in {-1,1,4,9,16,25}
        \draw[black!\INTENSITENOIR](1pt,\y/\ra) -- (-1pt,\y/\ra) node[left,fill=white]
        {$\ytext$};
      \end{scope}
      \draw[black!\INTENSITENOIR,domain={0:\XMAX},smooth,samples=50] plot (\x/\rx,{(\x-\rb)*(\x-\rb)/\ra});
      \node[circle,inner sep=2pt,red,draw, very thick] at (0,{0}) {};
      \node[circle,inner sep=2pt,red,draw, very thick] at (0,{-1/\ra}) {};

      \foreach \i in {1,...,\m} {
        \def\hi{(\i)/(\m+1)}
        \node[diamond,inner sep=2pt,blue,draw, very thick] at (0,{(-\hi)/(\ra))}) {};
      }

      \foreach \j in {1,...,\n} {
        \def\bj{1/(\m+1+\a*\j)}
        \node[regular polygon,regular polygon sides=4,inner sep=2pt,LimeGreen,draw, very thick] at ({\bj/2},{\bj*\bj/\ra}) {} ;
      }

      \def\hi{(1)/(\m+1)}
      \foreach \j in {2,4} {
        \def\bj{1/(\m+1+\a*\j)}
        \def\xij{(\hi)/(\bj)}
        \node[circle,inner sep=2pt,red,draw, very thick] (noeud1\j) at ({\xij/\rx},{\xij*\xij/\ra}) {};
      }
      \def\hi{(2)/(\m+1)}
      \foreach \j in {1,2,3} {
        \def\bj{1/(\m+1+\a*\j)}
        \def\xij{(\hi)/(\bj)}
        \node[circle,inner sep=2pt,red,draw, very thick] (noeud2\j) at ({\xij/\rx},{\xij*\xij/\ra}) {};
      }
      \def\hi{(3)/(\m+1)}
      \foreach \j in {3,4} {
        \def\bj{1/(\m+1+\a*\j)}
        \def\xij{(\hi)/(\bj)}
        \node[circle,inner sep=2pt,red,draw, very thick] (noeud3\j) at ({\xij/\rx},{\xij*\xij/\ra}) {};
      }
      \def\hi{(4)/(\m+1)}
      \foreach \j in {4,5} {
        \def\bj{1/(\m+1+\a*\j)}
        \def\xij{(\hi)/(\bj)}
        \node[circle,inner sep=2pt,red,draw, very thick] (noeud4\j) at ({\xij/\rx},{\xij*\xij/\ra}) {};
      }

      \draw[red!\INTENSITEROUGE,dashed](0,{0/\ra}) -- (0,{-1/\ra});
      \draw[red!\INTENSITEROUGE,dashed](0,{0/\ra}) -- (noeud45);
      \draw[red!\INTENSITEROUGE,dashed](0,{-1/\ra}) -- (noeud12);
      \draw[red!\INTENSITEROUGE,dashed](noeud12) -- (noeud14);
      \draw[red!\INTENSITEROUGE,dashed](noeud14) -- (noeud21);
      \draw[red!\INTENSITEROUGE,dashed](noeud21) -- (noeud22);
      \draw[red!\INTENSITEROUGE,dashed](noeud22) -- (noeud23);
      \draw[red!\INTENSITEROUGE,dashed](noeud23) -- (noeud33);
      \draw[red!\INTENSITEROUGE,dashed](noeud33) -- (noeud34);
      \draw[red!\INTENSITEROUGE,dashed](noeud34) -- (noeud44);
      \draw[red!\INTENSITEROUGE,dashed](noeud44) -- (noeud45);

    \end{tikzpicture}

    \begin{tikzpicture}[scale=1.6]
      \tikzstyle{axes}=[]
      \tikzstyle{important line}=[ultra thick]
      \tikzstyle{information text}=[rounded corners,fill=red!10,inner sep=1ex]
      \def\ra{0.5}
      \def\rx{1/20}
      \begin{scope}[style=axes]
        \draw[black!\INTENSITENOIR] (0,0) -- ({(0.3)/(\rx)},0) node[right] {};
        \draw[black!\INTENSITENOIR] (0,-1.1/\ra) -- (0,{(1)/(\ra)}) node[above] {};
        \foreach \x/\xtext in {0}
        \draw[black!\INTENSITENOIR](20*\x+\rb,1pt) -- (20*\x+\rb,-1pt) node[left,fill=white]
        {$\xtext$};
        \foreach \x/\xtext in {0.3}
        \draw[black!\INTENSITENOIR](20*\x+\rb,1pt) -- (20*\x+\rb,-1pt) node[below,fill=white]
        {$\xtext$};
        \foreach \y/\ytext in {-1,1}
        \draw[black!\INTENSITENOIR](1pt,\y/\ra) -- (-1pt,\y/\ra) node[left,fill=white]
        {$\ytext$};
        \draw[black!\INTENSITENOIR,domain={0:0.3},smooth,samples=50] plot (20*\x,{(\x-\rb)*(\x-\rb)/\ra}) node[above,fill=white]
        {$y=x^2$};
      \end{scope}

      \node[circle,inner sep=2pt,red,draw, very thick] at (0,{0}) {};
      \node[circle,inner sep=2pt,red,draw, very thick] at (0,{-1/\ra}) {};

      \foreach \i in {1,...,\m} {
        \def\hi{(\i)/(\m+1)}
        \node[diamond,inner sep=2pt,blue,draw, very thick] at (0,{(-\hi)/(\ra))}) {};
      }

      \foreach \j in {1,...,\n} {
        \def\bj{1/(\m+1+\a*\j)}
        \node[regular polygon,regular polygon sides=4,inner sep=2pt,LimeGreen,draw, very thick] at ({20*\bj},{\bj*\bj/\ra}) {} ;
      }
      \draw[red!\INTENSITEROUGE,dashed](0,{0/\ra}) -- (0,{-1/\ra});
      \draw[red!\INTENSITEROUGE,dashed](0,{0/\ra}) -- ({20*5/24},{1/\ra});
      \draw[red!\INTENSITEROUGE,dashed](0,{-1/\ra}) -- ({20*3/10},{(-448/1125)/\ra});

    \end{tikzpicture}
    
  \caption{Example of a 2-SSNMF instance constructed from the following SET-COVER instance: $n=5$, $m=4$, $C_1=\{2,4\}$, $C_2=\{1,2,3\}$, $C_3=\{3,4\}$ and $C_4=\{4,5\}$.
    The second picture is a zoom of the first picture on the $[0,0.3]\times [-1,1]$ box.
    Red circles correspond to the points of $M_3$, blue diamonds to the points of $M_1$ and green squares to the points of $M_2$.
    The red dashed lines are the edges of the convex hull of the points in $M_3$.}
  \label{examplecomplexity}
\end{figure}
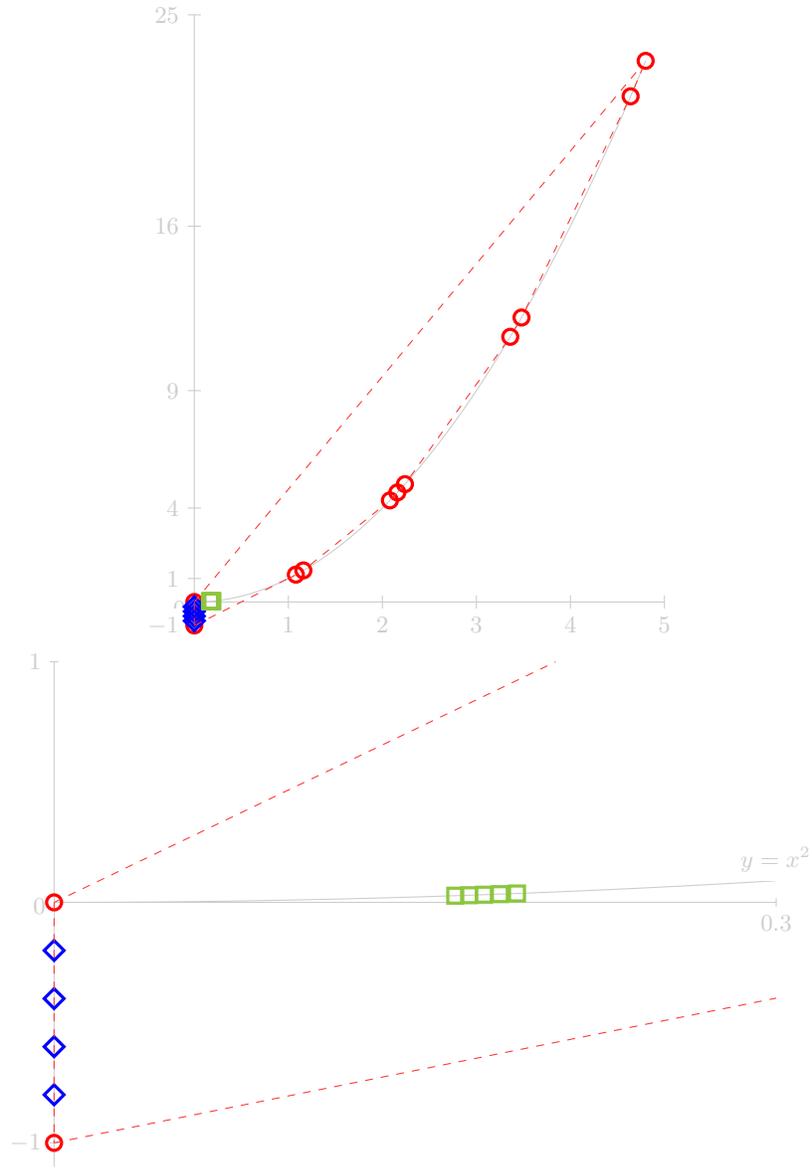

\bibliographystyle{splncs04}
\bibliography{mybib}

\begin{thebibliography}{10}
\providecommand{\url}[1]{\texttt{#1}}
\providecommand{\urlprefix}{URL }
\providecommand{\doi}[1]{https://doi.org/#1}

\bibitem{araujo_successive_2001}
Ara{\'u}jo, M.C.U., Saldanha, T.C.B., Galv{\~a}o, R.K.H., Yoneyama, T., Chame,
  H.C., Visani, V.: The successive projections algorithm for variable selection
  in spectroscopic multicomponent analysis. Chemometrics and Intelligent
  Laboratory Systems  \textbf{57}(2),  65--73 (2001)

\bibitem{arora_computing_2012}
Arora, S., Ge, R., Kannan, R., Moitra, A.: Computing a nonnegative matrix
  factorization \textendash{} provably. In: Proceedings of the Forty-Fourth
  Annual {{ACM}} Symposium on Theory of Computing. pp. 145--162 (2012)

\bibitem{bioucas2012hyperspectral}
Bioucas-Dias, J.M., Plaza, A., Dobigeon, N., Parente, M., Du, Q., Gader, P.,
  Chanussot, J.: Hyperspectral unmixing overview: Geometrical, statistical, and
  sparse regression-based approaches. IEEE journal of selected topics in
  applied earth observations and remote sensing  \textbf{5}(2),  354--379
  (2012)

\bibitem{cohen_nonnegative_2019-1}
Cohen, J.E., Gillis, N.: Nonnegative {{Low}}-rank {{Sparse Component
  Analysis}}. In: {{IEEE International Conference}} on {{Acoustics}},
  {{Speech}} and {{Signal Processing}} ({{ICASSP}}). pp. 8226--8230 (2019)

\bibitem{el2010safe}
El~Ghaoui, L., Viallon, V., Rabbani, T.: Safe feature elimination in sparse
  supervised learning technical report no. Tech. rep., UC/EECS-2010-126, EECS
  Dept., University of California at Berkeley (2010)

\bibitem{fu2019nonnegative}
Fu, X., Huang, K., Sidiropoulos, N.D., Ma, W.K.: Nonnegative matrix
  factorization for signal and data analytics: Identifiability, algorithms, and
  applications. IEEE Signal Processing Magazine  \textbf{36}(2),  59--80 (2019)

\bibitem{gillis_successive_2014}
Gillis, N.: Successive {{Nonnegative Projection Algorithm}} for {{Robust
  Nonnegative Blind Source Separation}}. SIAM Journal on Imaging Sciences pp.
  1420--1450 (2014)

\bibitem{gillis_why_2014}
Gillis, N.: The why and how of nonnegative matrix factorization.
  Regularization, Optimization, Kernels, and Support Vector Machines
  \textbf{12}(257),  257--291 (2014)

\bibitem{hoyer2002non}
Hoyer, P.O.: Non-negative sparse coding. In: Proceedings of the 12th IEEE
  Workshop On Neural Networks for Signal Processing. pp. 557--565 (2002)

\bibitem{hoyer2004non}
Hoyer, P.O.: Non-negative matrix factorization with sparseness constraints.
  Journal of machine learning research  \textbf{5},  1457--1469 (2004)

\bibitem{kim_sparse_2007}
Kim, H., Park, H.: Sparse non-negative matrix factorizations via alternating
  non-negativity-constrained least squares for microarray data analysis.
  Bioinformatics  \textbf{23}(12),  1495--1502 (2007)

\bibitem{kumar_fast_2013}
Kumar, A., Sindhwani, V., Kambadur, P.: Fast {{Conical Hull Algorithms}} for
  {{Near}}-separable {{Non}}-negative {{Matrix Factorization}}. In: Proceedings
  of the 30th {{International Conference}} on {{Machine Learning}} (2013)

\bibitem{lee1999learning}
Lee, D.D., Seung, H.S.: Learning the parts of objects by non-negative matrix
  factorization. Nature  \textbf{401}(6755),  788--791 (1999)

\bibitem{ma2013signal}
Ma, W.K., Bioucas-Dias, J.M., Chan, T.H., Gillis, N., Gader, P., Plaza, A.J.,
  Ambikapathi, A., Chi, C.Y.: A signal processing perspective on hyperspectral
  unmixing: Insights from remote sensing. IEEE Signal Processing Magazine
  \textbf{31}(1),  67--81 (2014)

\bibitem{naanaa2005blind}
Naanaa, W., Nuzillard, J.M.: Blind source separation of positive and partially
  correlated data. Signal Processing  \textbf{85}(9),  1711--1722 (2005)

\bibitem{nadisicexact}
Nadisic, N., Vandaele, A., Gillis, N., Cohen, J.E.: Exact {Sparse}
  {Nonnegative} {Least} {Squares}. In: {{IEEE International Conference}} on
  {{Acoustics}}, {{Speech}} and {{Signal Processing}} ({{ICASSP}}). pp. 5395 --
  5399 (2020)

\bibitem{natarajan1995sparse}
Natarajan, B.K.: Sparse approximate solutions to linear systems. SIAM journal
  on computing  \textbf{24}(2),  227--234 (1995)

\bibitem{sun2011underdetermined}
Sun, Y., Xin, J.: Underdetermined sparse blind source separation of nonnegative
  and partially overlapped data. SIAM Journal on Scientific Computing
  \textbf{33}(4),  2063--2094 (2011)

\bibitem{vavasis_complexity_2010}
Vavasis, S.A.: On the {{Complexity}} of {{Nonnegative Matrix Factorization}}.
  SIAM Journal on Optimization  \textbf{20}(3),  1364--1377 (2010)

\bibitem{zhu2017hyperspectral}
Zhu, F.: Hyperspectral unmixing: ground truth labeling, datasets, benchmark
  performances and survey. arXiv preprint arXiv:1708.05125  (2017)

\bibitem{fyzhu2014hyperspectraldata}
Zhu, F., Wang, Y., Xiang, S., Fan, B., Pan, C.: Structured sparse method for
  hyperspectral unmixing. ISPRS Journal of Photogrammetry and Remote Sensing
  \textbf{88},  101--118 (2014)

\end{thebibliography}
\end{document}